\documentclass[12pt, final]{l4dc2022}

\title[Safe Autonomous Navigation for Systems with Learned SE(3) Hamiltonian Dynamics]{Safe Autonomous Navigation for Systems with Learned SE(3) Hamiltonian Dynamics}

\usepackage{times}

\usepackage{amsmath,amssymb,amsfonts, dsfont} 		
\usepackage{mathtools} 
\usepackage{graphicx,wrapfig,adjustbox}

\author{\Name{Zhichao Li}\thanks{These authors contributed equally.} \Email{zhichaoli@ucsd.edu}\\
 \Name{Thai Duong}\footnotemark[1] \Email{tduong@ucsd.edu}\\
 \Name{Nikolay Atanasov} \Email{natanasov@ucsd.edu} \\
 \addr Department of Electrical and Computer Engineering, University of California San Diego, La Jolla, CA 92093%
}

\newtheorem{problem}{Problem}
\newtheorem*{problem*}{Problem}
\newcommand{\calD}{{\cal D}}

\newcommand{\calF}{{\cal F}}

\newcommand{\calH}{{\cal H}}

\newcommand{\calL}{{\cal L}}

\newcommand{\calO}{{\cal O}}
\newcommand{\calP}{{\cal P}}

\newcommand{\calS}{{\cal S}}
\newcommand{\calT}{{\cal T}}
\newcommand{\calU}{{\cal U}}

\newcommand{\frakp}{{\mathfrak{p}}}
\newcommand{\frakq}{{\mathfrak{q}}}

\newcommand{\bfa}{\mathbf{a}}

\newcommand{\bfc}{\mathbf{c}}
\newcommand{\bfd}{\mathbf{d}}
\newcommand{\bfe}{\mathbf{e}}
\newcommand{\bff}{\mathbf{f}}
\newcommand{\bfg}{\mathbf{g}}

\newcommand{\bfp}{\mathbf{p}}
\newcommand{\bfq}{\mathbf{q}}
\newcommand{\bfr}{\mathbf{r}}

\newcommand{\bfu}{\mathbf{u}}
\newcommand{\bfv}{\mathbf{v}}
\newcommand{\bfw}{\mathbf{w}}
\newcommand{\bfx}{\mathbf{x}}
\newcommand{\bfy}{\mathbf{y}}

\newcommand{\bfzeta}{{\boldsymbol{\zeta}}}

\newcommand{\bftheta}{{\boldsymbol{\theta}}}

\newcommand{\bfpi}{{\boldsymbol{\pi}}}

\newcommand{\bfomega}{{\boldsymbol{\omega}}}

\newcommand{\bfell}{{\boldsymbol{\ell}}}

\newcommand{\bfB}{\mathbf{B}}

\newcommand{\bfG}{\mathbf{G}}

\newcommand{\bfI}{\mathbf{I}}
\newcommand{\bfJ}{\mathbf{J}}
\newcommand{\bfK}{\mathbf{K}}

\newcommand{\bfM}{\mathbf{M}}

\newcommand{\bfR}{\mathbf{R}}

\newcommand{\bbR}{\mathbb{R}}

\newcommand{\prl}[1]{\left(#1\right)}
\newcommand{\brl}[1]{\left[#1\right]}
\newcommand{\crl}[1]{\left\{#1\right\}}

\DeclareMathOperator{\tr}{tr}
\DeclareMathOperator*{\argmax}{argmax}   

\DeclareMathOperator*{\diag}{diag}

\DeclarePairedDelimiterX{\norm}[1]{\lVert}{\rVert}{#1}

\newcommand*{\intF}{\text{int}(\calF)}

\newcommand*{\lpg}{\bar{\bfg}}
\newcommand*{\LS}{\mathcal{LS}}

\newcommand*{\nInput}{p}
\newcommand*{\nOutput}{m}

\newcommand{\NEWTD}[1]{{\color{black}#1}}
\newcommand{\NEWWTD}[1]{{\color{black}#1}}
\newcommand{\NEWZL}[1]{{\color{black}#1}}

\begin{document}

\maketitle

\begin{abstract}
Safe autonomous navigation in unknown environments is an important problem for mobile robots. This paper proposes techniques to learn the dynamics model of a mobile robot from trajectory data and synthesize a tracking controller with safety and stability guarantees. The state of a rigid-body robot usually contains its position, orientation, and generalized velocity and satisfies Hamilton's equations of motion. Instead of a hand-derived dynamics model, we use a dataset of state-control trajectories to train a translation-equivariant nonlinear Hamiltonian model represented as a neural ordinary differential equation (ODE) network. The learned Hamiltonian model is used to synthesize an energy-shaping passivity-based controller and derive conditions which guarantee safe regulation to a desired reference pose. We enable adaptive tracking of a desired path, subject to safety constraints obtained from obstacle distance measurements. The trade-off between the robot's energy and the distance to safety constraint violation is used to adaptively govern a reference pose along the desired path. Our safe adaptive controller is demonstrated on a simulated hexarotor robot navigating in an unknown environments.
\end{abstract}

\begin{keywords}%
dynamics learning, neural ODE network, reference governor, safe tracking control%
\end{keywords}


\section{Introduction}
Designing controllers that handle safety constraints and guarantee stability is an important problem in safety-critical applications of robotics, such as autonomous driving \citep{CBF_ames2014control, shalev2016safe}, locomotion \citep{CBF_ames2014rapidly} or medical robotics \citep{yip2014model}. Safety depends on the system states, governed by the system dynamics, and the environment constraints. This leads to two requirements for designing provably safe controllers: 1) the availability of an accurate dynamics model and 2) the satisfaction of time-varying safety constraints that are only known at runtime.




\NEWTD{
The first requirement has motivated machine learning techniques for system dynamics learning, e.g. based on Gaussian processes \citep{deisenroth2015gp,kabzan2019learning} or neural networks \citep{raissi2018multistep,chua2018deep}. For physical systems, recent works \citep{lutter2019deepunderactuated, zhong2020symplectic, duong21hamiltonian} design the model architecture to encode a Lagrangian or Hamiltonian formulation of robot dynamics \citep{lurie2013analytical,HolmBook}, which a black-box model might struggle to infer. \cite{zhong2020symplectic} use a differentiable neural ODE solver \citep{chen2018neural} to generate predicted state trajectories in a Hamiltonian formulation. A loss function is back-propagated through the ODE solver to update the model parameters. \cite{duong21hamiltonian} extend this approach by imposing both Hamiltonian dynamics and $SE(3)$ pose constraints on the ODE structure. A Hamiltonian-based model architecture also simplifies the design of stable regulation or tracking control by energy shaping \citep{zhong2020symplectic, duong2021learning, duong21hamiltonian}. The key idea of energy-based shaping, known as interconnection and damping assignment passivity-based control (IDA-PBC) \citep{van2014port}, is to inject additional energy via the control input into the system to achieve a desired total energy, minimized at a desired set point. 
%

The second requirement related to safety guarantees has gained significant attention in planning and control. Model predictive control (MPC) methods \citep{borrelli_MPC_book,NMPC_book,MPC_Automatica_Camacho2006_nlin,MPC_mayne2000constrained} include safety constraints in an optimization problem, which is typically solved by discretizing time with linearized dynamics.
Reachability-based techniques \citep{herbert2017fastrack, Funnel_lib, RTD_kousik2018bridging} work directly for nonlinear systems and offer strong safety guarantees but have high computation cost and scalability issues for high-dimensional systems.
%
Control barrier functions (CBFs) with quadratic programming (QP) \citep{CBF_ames2014control,CBF_ames2017TAC,CBF_ames2019ECC} offer an elegant and efficient framework for real-time safe control synthesis. However, construct a valid CBF \citep{CBF_ames2019ECC} that guarantees the feasibility of the QP problem \citep{CBF_xu2018constrained} at all times is challenging. 
Given a stabilizing regulation controller, reference governor techniques \citep{RG_bemporad1998reference,RG_kolmanovsky2014ACC_tutorial,RG_garone2016_ERG} maintain a virtual governor system to adaptively generate a regulation so that the system follows reference commands safely. Recent work \citep{RG_Omur_ICRA17, Li_SafeControl_ICRA20} achieves safe navigation in unknown environments but is limited to feedback-linearizable systems.

}



In this paper, we consider both requirements for rigid-body robot systems, whose states are described by their $SE(3)$ pose and generalized velocity. We assume that the robot dynamics are unknown but, as a physical system, satisfy Hamilton's equations of motion over the $SE(3)$ manifold. We consider a training set of state-control trajectories, from past experiments or collected by a human operator, and seek to safely track a desired position path with safety constraints obtained online from distance-to-obstacles measurements. We learn a \NEWWTD{$SE(3)$} Hamiltonian model of the system dynamics using a neural ODE network \citep{duong21hamiltonian}. \NEWTD{As the robot dynamics are equivariant to translation, we offset the trajectories to start from the origin and train a translation-equivariant Hamiltonian neural ODE model. The Hamiltonian structure of the learned model offers an energy-based regulation controller with the total energy of the system viewed as a Lyapunov function}. This, in turn, enables us to enforce safety constraints using reference governor techniques without the \NEWZL{need} to linearize the system dynamics. \NEWTD{Inspired by constraint embedding techniques
\NEWZL{\citep{RG_garone2016_ERG}},
we impose safety constraints, based on the sensor measurements, on the Lyapunov function. We use the trade-off between the distance from constraint violation and the system energy level to regulate a reference governor and achieve safe and stable position tracking in an unknown environment.}

\textbf{Contributions.}
In summary, the contributions of this paper are 1) \NEWWTD{a translation-equivariant $SE(3)$ Hamiltonian dynamics learning approach} and 2) a tracking control design for $SE(3)$ Hamiltonian systems with stability and safety guarantees. Our dynamics learning and tracking control techniques are demonstrated on a simulated hexarotor robot using a depth sensor to navigate in unknown environments.

\section{Problem Statement}
\label{sec:problem}

Consider a robot modeled as a rigid body with position $\bfp \in \bbR^3$, orientation $\bfR\in SO(3)$, body-frame linear velocity $\bfv \in \bbR^3$, and body-frame angular velocity $\bfomega \in \bbR^3$. Let $\frakq = [\bfp^\top\;\; \bfr_1^\top\;\; \bfr_2^\top\;\; \bfr_3^\top]^\top$ $\in \bbR^{12}$ denote the robot's generalized coordinates, where $\bfr_1$, $\bfr_2$, $\bfr_3 \in \bbR^3$ are the rows of the rotation matrix $\bfR$. Let $\bfzeta = [\bfv^\top\;\;\bfomega^\top]^\top \in \bbR^6$ denote the robot's generalized velocity. The generalized momentum $\frakp$ of the system is defined as:
\begin{equation}
\label{eq:momenta_Mtwist}
\frakp = \bfM(\frakq)\bfzeta \in \bbR^6,
\end{equation} 
where $\bfM(\frakq) \succ 0$ denotes a positive-definite $6 \times 6$ generalized mass matrix. Let $\bfx = (\frakq, \frakp) \in \bbR^{18}$ be the robot state. The Hamiltonian, $\mathcal{H}(\mathbf\frakq, \mathbf\frakp)$, captures the total energy of the system as the sum of the kinetic energy $\calT(\mathbf\frakq, \mathbf\frakp) =  \frac{1}{2}\mathbf\frakp^\top \bfM(\mathbf\frakq)^{-1} \mathbf\frakp$ and the potential energy $\calU(\frakq)$:
\begin{equation}
\label{eq:hamiltonian_def}
\mathcal{H}(\mathbf\frakq, \mathbf\frakp) = \calT(\mathbf\frakq, \mathbf\frakp) +\calU(\mathbf\frakq) =  \frac{1}{2}\mathbf\frakp^\top \bfM(\mathbf\frakq)^{-1} \mathbf\frakp + \calU(\mathbf\frakq).
\end{equation}
As a mechanical system, the time evolution of the state $\bfx$ is governed by Hamilton's equations of motion \citep{lee2017global, duong21hamiltonian}:
\begin{equation} \label{eq:Ham_sys_pf_nlin}
\dot{\bfx} = \bff(\bfx) + \bfG(\bfx) \bfu, \quad \bfx(t_0) = \bfx_0,
\end{equation}
where $\bfu \in \bbR^\nInput$ is the control input, 
\NEWZL{$\bff(\bfx) = \begin{bmatrix}
\bf0 & \mathbf\frakq^{\times} \\
-\mathbf\frakq^{\times\top} & \mathbf\frakp^{\times} 
\end{bmatrix}
\begin{bmatrix}
\nabla_{\frakq} \calH \\
\nabla_{\frakp} \calH 
\end{bmatrix}$, }
$\bfG(\bfx) = \begin{bmatrix} \bf0 \\ \bfB(\mathbf\frakq) \end{bmatrix}$, and $\bfB(\mathbf\frakq)\in \bbR^{6\times \nInput}$ is an input gain matrix. The operators $\mathbf\frakq^{\times}$, $\mathbf\frakp^{\times}$, and the hat map $\hat{\bfw}$ for $\bfw \in \mathbb{R}^3$ are defined as:
\begin{equation*}
\mathbf\frakq^{\times} = \begin{bmatrix}
\bfR^\top\!\!\!\! & \bf0 & \bf0 & \bf0 \\
\bf0 & \hat{\bfr}_1^\top & \hat{\bfr}_2^\top & \hat{\bfr}_3^\top
\end{bmatrix}^\top\!\!\!\!, \quad \mathbf\frakp^{\times} = \begin{bmatrix} \mathbf\frakp_{\bfv}\\\mathbf\frakp_{\bfomega}\end{bmatrix}^{\times} \!\!\!\!= \begin{bmatrix}
\bf0 & \hat{\mathbf\frakp}_{\bfv}\\
\hat{\mathbf\frakp}_{\bfv} & \hat{\mathbf\frakp}_{\bfomega}
\end{bmatrix}, \quad \hat{\bfw} = \begin{bmatrix}
0 & -w_3 & w_2 \\
w_3 & 0 & -w_1 \\
-w_2 & w_1 & 0 
\end{bmatrix}.
\end{equation*}

We consider the case that the parameters of the Hamiltonian dynamics model in \eqref{eq:Ham_sys_pf_nlin}, including the mass $\bfM(\frakq)$, potential energy $\calU(\frakq)$, and input matrix $\bfB(\frakq)$, are unknown. Instead, we are given a trajectory dataset $\calD = \{t^{(i)}_{0:N}, \frakq^{(i)}_{0:N},\bfzeta^{(i)}_{0:N}, \bfu^{(i)}\}_{i = 1}^{D}$ consisting of $D$ sequences of generalized coordinates and velocities $(\frakq^{(i)}_{0:N},\bfzeta^{(i)}_{0:N})$ at times $t_0^{(i)} < t_1^{(i)} < \ldots < t_N^{(i)}$, collected by applying a constant control input $\bfu^{(i)}$ to the system with initial condition $(\frakq^{(i)}_{0},\bfzeta^{(i)}_{0})$. We aim to learn the dynamics from the data set $\calD$ and design a control policy $\bfu = \bfpi(\bfx)$ such that the robot follows a desired reference path without violating safety constraints in an unknown environment. Let $\calO \subset \bbR^3$ and $\calF \coloneqq \bbR^{3} \setminus \calO$ denote the unsafe (e.g., obstacle) set and the safe (obstacle-free) set, respectively. Denote the interior of $\calF$ as $\intF$. We assume that $\calO$ is not known a priori but the robot can sense the distance $\bar{d}(\bfp, \calO)$ from its position $\bfp$ to $\calO$ locally with a limited sensing range $\beta > 0$:
\begin{equation} \label{eq:dyO}
\bar{d}(\bfp , \calO) \coloneqq \min \crl{d(\bfp, \calO), \beta},
\end{equation}
where $d(\bfp, \calO) \coloneqq \inf_{\bfa \in \calO} \norm{\bfp - \bfa}$ denotes the Euclidean distance from $\bfp$ to the set $\calO$. The safe autonomous navigation problem considered in this paper is summarized below.

\begin{problem}
\label{prob:main_prob}
Let $\calD = \{t^{(i)}_{0:N}, \frakq^{(i)}_{0:N},\bfzeta^{(i)}_{0:N}, \bfu^{(i)}\}_{i = 1}^{D}$ be a training dataset of state-control trajectories obtained from a robot with unknown Hamiltonian dynamics in \eqref{eq:Ham_sys_pf_nlin}. Let $\bfr: \brl{0,1} \mapsto \text{Int}\prl{\calF}$ be a continuous function specifying a desired position reference path for the robot. Assume that the reference path starts at the initial robot position at time $t_0$, i.e., $\bfr(0) = \bfp(t_0) \in \text{Int}\prl{\calF}$. Using local distance observations $\bar{d}(\bfp(t),\calO)$ of the unsafe set $\calO$ in an unknown environment, design a control policy $\bfpi(\bfx)$ so that the position $\bfp(t)$ of the closed-loop system converges asymptotically to $\bfr(1)$, while remaining safe, i.e., $\bfp(t) \in \calF$ for all $t \geq t_0$. 
\end{problem}

\section{Learning Hamiltonian Dynamics on the $SE(3)$ Manifold}
\label{subsec:ham_dyn_learning}
\subsection{Training a translation-equivariant $SE(3)$ Hamiltonian dynamics model}
\label{subsec:ham_neural_ode}
The training dataset $\calD = \{t^{(i)}_{0:N}, \frakq^{(i)}_{0:N},\bfzeta^{(i)}_{0:N}, \bfu^{(i)}\}_{i = 1}^{D}$ can be obatined by measuring the generalized coordinates and velocities of the system at times $t_0^{(i)} < t_1^{(i)} < \ldots < t_N^{(i)}$ using an odometry algorithm \citep{vio_benchmark,OdometrySurvey} or a motion capture system. The control input $\bfu^{(i)}$ can be generated by manually driving the robot or using an  existing controller. Since the system dynamics \NEWZL{do} not change if we shift the position $\bfp$ to any points in the world frame,  we offset the trajectories in the dataset $\calD$ so that they start from the position $\bf0$ and learn the system dynamics around the origin. \NEWZL{This is sufficient for control purposes, e.g. using the control design in Sec. \ref{sec:ham_control}, the input driving the system from state $\bfx$ with position $\bfp$ to desired state $\bfx^*$ with position $\bfp^*$ is the same as the one driving the system from state $\bfx$ with position $\bf0$ to the desired state $\bfx^*$ with the offset position $\bfp^* - \bfp$.}

Since the \NEWZL{momentum} $\frakp$ is not directly available from the dataset $\calD$, we use the time derivative of the generalized velocity, derived from \eqref{eq:momenta_Mtwist}:
\begin{equation}
\label{eq:hamiltonian_zetadot}
\dot{\bfzeta} =  \prl{ \frac{d}{dt} \bfM^{-1}(\mathbf\frakq) }\mathbf\frakp + \bfM^{-1}(\mathbf\frakq)\dot{\mathbf\frakp}.
\end{equation}
Eq.~\eqref{eq:Ham_sys_pf_nlin} and \eqref{eq:hamiltonian_zetadot} 
\NEWZL{describe the Hamiltonian dynamics } 
with unknown inverse generalized mass matrix $\bfM(\frakq)^{-1}$, input matrix $\bfB(\frakq)$, and potential energy $\calU(\frakq)$, which we aim to approximate by three neural networks $\bfM_\bftheta(\frakq)^{-1}, \bfB_\bftheta(\frakq)$ and $\calU_\bftheta(\frakq)$, respectively, with parameters $\bftheta$.

To optimize the parameters $\bftheta$, we use a neural ODE framework that encodes the Hamiltonian dynamics \eqref{eq:Ham_sys_pf_nlin} and \eqref{eq:hamiltonian_zetadot} with $\bfM_\bftheta(\frakq), \bfB_\bftheta(\frakq)$ and $\calU_\bftheta(\frakq)$ in the network structure (Fig. \ref{fig:fa_drone_learning}(a)). The forward pass rolls out the Hamiltonian dynamics \eqref{eq:Ham_sys_pf_nlin} and \eqref{eq:hamiltonian_zetadot} with the neural networks $\bfM_\bftheta(\frakq), \bfB_\bftheta(\frakq)$ and $\calU_\bftheta(\frakq)$ using a neural ODE solver \citep{chen2018neural} to obtain a predicted sequence $(\bar{\frakq}^{(i)}_{0:N}, \bar{\bfzeta}^{(i)}_{0:N})$ at times $t_0^{(i)} < t_1^{(i)} < \ldots < t_N^{(i)}$ for each $i = 1,\ldots, D$. The loss function is defined as $\calL = \sum_{i = 1}^D \sum_{n = 1}^N c(\frakq^{(i)}_{0:N},\bfzeta^{(i)}_{0:N}, \bar{\frakq}^{(i)}_{0:N}, \bar{\bfzeta}^{(i)}_{0:N})$, where the distance metric $c$ is defined as the sum of position, orientation, and velocity errors:
\begin{equation}
c \prl{\frakq^{(i)}_{0:N},\bfzeta^{(i)}_{0:N}, \bar{\frakq}^{(i)}_{0:N}, \bar{\bfzeta}^{(i)}_{0:N}} = c_{\bfp}(\bfp,\bar{\bfp}) + c_{\bfR}(\bfR,\bar{\bfR}) + c_{\bfzeta}(\bfzeta,\bar{\bfzeta}),
\end{equation}
where $c_{\bfp}(\bfp,\bar{\bfp}) = \| \bfp - \bar{\bfp}\|^2_2$, $c_{\bfzeta}(\bfzeta,\bar{\bfzeta}) = \| \bfzeta - \bar{\bfzeta}\|^2_2$, and $c_{\bfR}(\bfR,\bar{\bfR}) = \|\left(\log (\bar{\bfR} \bfR^\top)\right)^{\vee} \|_2^2$. The $\log$-map $\log (\cdot): SE(3) \mapsto \mathfrak{so}(3)$ is the inverse of the exponential map, returning a skew-symmetric matrix in $\mathfrak{so}(3)$ from a rotation matrix in $SE(3)$, and the $\vee$-map $(\cdot)^\vee : \mathfrak{so}(3) \mapsto \bbR^3$ is the inverse of the hat map $\hat{(\cdot)}$ in Sec. \ref{sec:problem}. The parameters $\bftheta$ are optimized via gradient descent by back-propagating the loss through the neural ODE solver. This is done using the adjoint method, where the gradient $\partial \calL/\partial \bftheta$ is calculated by a single call to a reverse-time ODE starting from $t = t_N$ at $(\frakq_N^{(i)}, \bfzeta_N^{(i)})$.

\subsection{Evaluation of the Hamiltonian dynamics model of a simulated hexarotor}
\label{subsec:ham_dyn_learning_evaluation}

\begin{figure}[t]
\centering
\begin{subfigure}
        \centering
        \includegraphics[height=30mm]{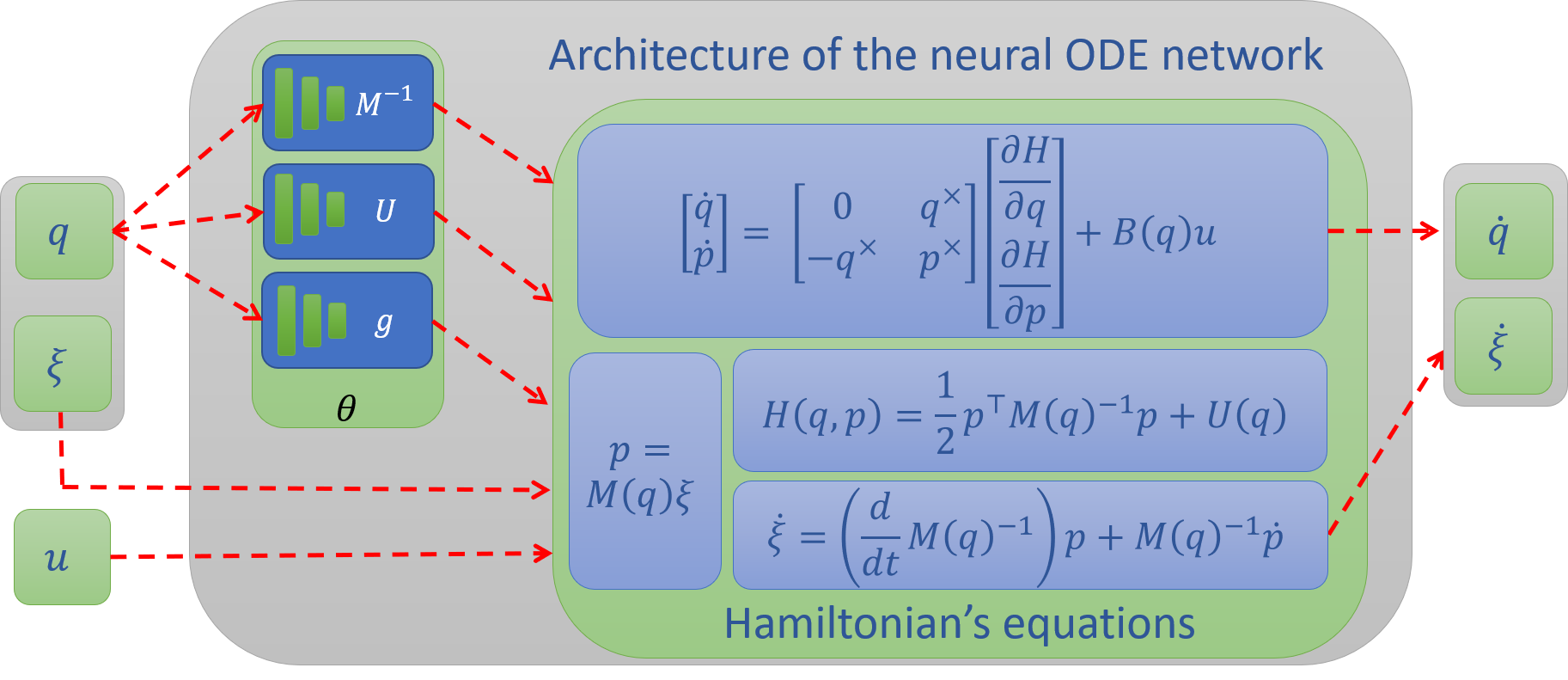}
        \label{fig:net_arch}
\end{subfigure}%
\hfill
\begin{subfigure}
        \centering
        \includegraphics[height=27mm]{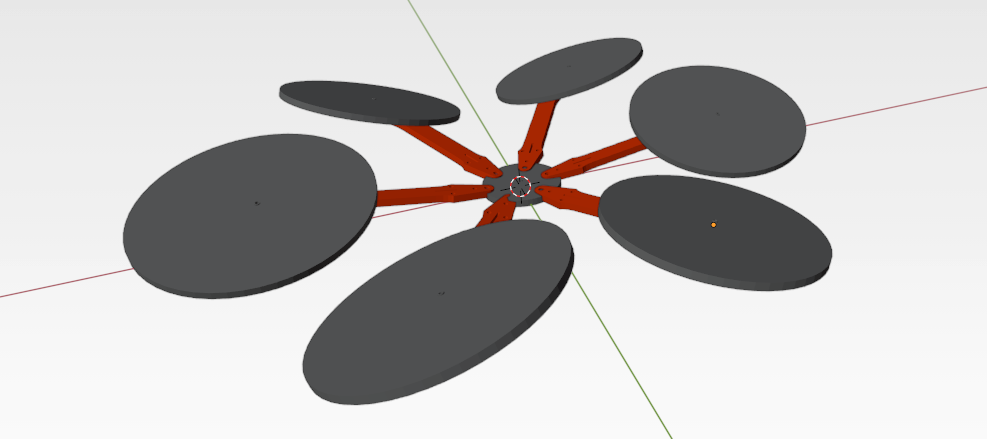}
        \label{fig:hexarotor}
\end{subfigure}%

\begin{subfigure}
        \centering
        \includegraphics[height=28mm]{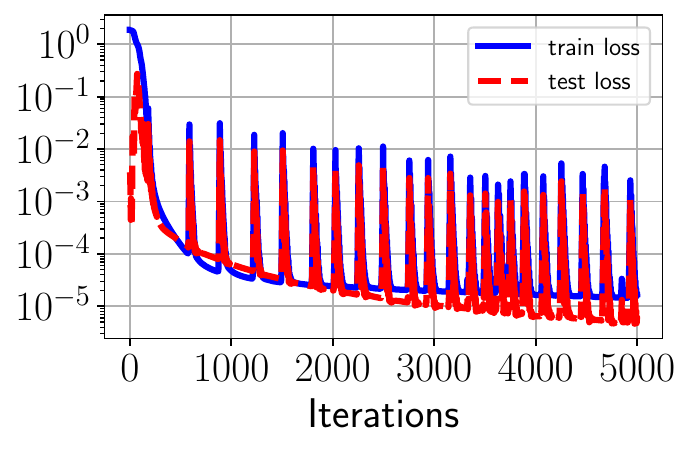}
        \label{fig:loss_log}
\end{subfigure}%
\hfill%
\begin{subfigure}
        \centering
        \includegraphics[height=28mm]{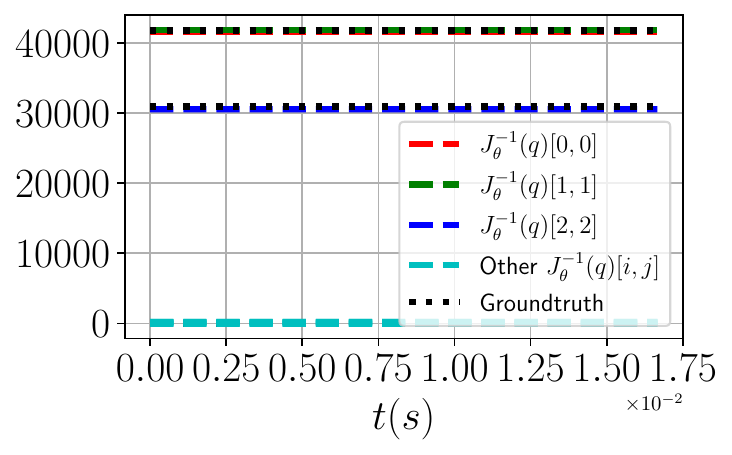}
        \label{fig:learned_inertia}
\end{subfigure}%
\hfill%
\begin{subfigure}
        \centering
        \includegraphics[height=28mm]{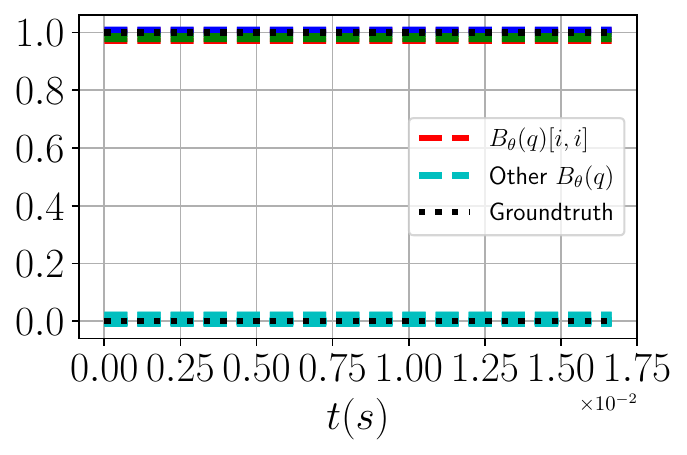}
        \label{fig:learned_g}
\end{subfigure}%
\caption{$SE(3)$ Hamiltonian neural ODE network: (a) network architecture, (b) simulated hexarotor for evaluation, (c) training loss, (d) learned inverse inertia $\bfJ_\bftheta(\frakq)^{-1}$, and (e) learned input matrix $\bfB_\bftheta(\frak{q})$ along a test trajectory, evaluated on the simulated hexarotor.}
\label{fig:fa_drone_learning}
\end{figure}

We consider a simulated hexarotor unmanned aerial vehicle (UAV) (Fig. \ref{fig:fa_drone_learning}(b)) with fixed-tilt rotors pointing in different directions (\cite{hexarotor}) modeled as a fully-actuated rigid body with mass $m = 0.027$ and inertia matrix $\bfJ = 10^{-5}\text{diag}([2.4, 2.4, 3.2])$. The robot's ground-truth dynamics satisfy Hamilton's equations in \eqref{eq:Ham_sys_pf_nlin} with generalized mass $\bfM(\mathbf\frakq) = \diag(m\bfI, \bfJ)$, potential energy $\calU(\mathbf\frakq) = mgz$, and the input matrix $\bfB(\mathbf\frakq)= \bfI$. The control input $\bfu$ is a 6-dimensional wrench, including a 3-dimensional force and a 3-dimensional torque.
Since the mass $m$ of the hexarotor can be easily measured, we assume the mass $m$ is known, leading to a known potential energy $\calU(\frakq) = mg\begin{bmatrix}
0 & 0 & 1
\end{bmatrix} \bfp$, where $\bfp$ is the UAV position and $g\approx 9.8ms^{-2}$ is the gravitational acceleration. We approximate the inverse generalized mass matrix by $\bfM_\bftheta(\mathbf\frakq)^{-1} =\diag(m^{-1}\bfI, \bfJ_{\bftheta}^{-1}(\mathbf\frakq))$ and learn the inverse inertia matrix $\bfJ_{\bftheta}(\mathbf\frakq)^{-1}$ and the input matrix $\bfB_\bftheta(\frakq)$ from data.

\NEWWTD{As the dynamics model is translation-invariant, }we mimic manual flights in an area free of obstacles using a PID controller and drive the hexarotor from a random initial pose \NEWWTD{near the origin} to a desired \NEWZL{pose}, generating $18$ one-second trajectories. We shift the trajectories to start from the origin and create a dataset $\mathcal{D} = \{t_{0:N}^{(i)},\mathbf\frakq_{0:N}^{(i)}, \bfzeta_{0:N}^{(i)}, \bfu^{(i)})\}_{i=1}^D$ with $N = 5$ and $D = 432$. The Hamiltonian-based neural ODE network\footnote{Code: \url{https://thaipduong.github.io/SE3HamDL/}} is trained with the dataset $\calD$, as described in Sec. \ref{subsec:ham_dyn_learning}, for 5000 iterations and learning rate $10^{-3}$. Fig. \ref{fig:fa_drone_learning}(c) shows the loss function during training. Note that if we scale $\bfM_\bftheta(\frakq)$ and the input matrix $\bfB(\frakq)$ by a constant $\gamma$, the dynamics of $(\frakq, \bfzeta)$ in \eqref{eq:Ham_sys_pf_nlin} and \eqref{eq:hamiltonian_zetadot} do not change. Fig. \ref{fig:fa_drone_learning}(d) and \ref{fig:fa_drone_learning}(e) plot the scaled version of the learned inverse mass $\bfJ_\bftheta(\frakq)^{-1}$ and the input matrix $\bfB_\bftheta(\frakq)$, converging to the constant ground truth values.

\section{Safe Tracking using a Reference Governor}

In this section, we first describe a passivity-based regulation controller for arbitrary pose stabilization in Sec.~\ref{sec:ham_control}. We derive sufficient conditions for safety based on an invariant level set of the closed-loop system's Hamiltonian. In Sec.~\ref{sec:rg_ctrl}, we propose a reference governor control policy to adaptively generate a regulation pose along the desired path and achieve safe navigation.

\subsection{Passivity-based control for learned Hamiltonian dynamics}
\label{sec:ham_control}

Given the learned model of the system dynamics:
\begin{equation} \label{eq:ols_nln}
\dot{\bfx} = \bff_\bftheta(\bfx) + \bfG_\bftheta(\bfx) \bfu, \quad \bfx(t_0) = \bfx_0,
\end{equation}
we want to find a control policy that stabilizes the system to a desired equilibrium $\bfx^* \coloneqq (\mathbf\frakq^*, \bf0)$ with desired generalized coordinates $\frakq^* = (\bfp^*, \bfR^*)$ and zero momentum $\frakp^* = \bf0$.
We design a control policy $\bfu = \bfpi(\bfx,\bfx^*)$ to shape the total energy (Hamiltonian) of the closed-loop system so that it achieves a minimum at the desired state $\bfx^* = (\mathbf\frakq^*, \bf0)$. By injecting energy into the system through the controller $\bfu = \bfpi(\bfx,\bfx^*)$, we aim to achieve the following desired Hamiltonian:
\begin{equation}
\label{eq:desired_hamiltonian}
    \mathcal{H}_d (\bfx, \bfx^*) =  \frac{1}{2}k_\bfp(\bfp - \bfp^*)^\top(\bfp - \bfp^*) + \frac{1}{2} k_{\bfR}\tr(\bfI - \bfR^{*\top}\bfR) + \frac{1}{2}(\mathbf\frakp-\mathbf\frakp^*)^\top\bfM_\bftheta^{-1}(\mathbf\frakq)(\mathbf\frakp-\mathbf\frakp^*),
\end{equation}
where $k_\bfp$ and $k_{\bfR}$ are positive gains. 
\NEWZL{
We use interconnection and damping assignment passivity-based control (IDA-PBC) \citep{van2014port}: }
\begin{equation}
    \label{eq:stab_ctrl_law}
    \bfu = \bfB_\bftheta^{\dagger}(\mathbf\frakq)\left(\mathbf\frakq^{\times\top} \NEWZL{\nabla_\frakq \calU_\bftheta(\frakq)}  -\mathbf\frakp^{\times}\bfM_\bftheta^{-1}(\mathbf\frakq)\mathbf\frakp - \bfe(\mathbf\frakq,\mathbf\frakq^*)\right) +\prl{- \bfB_\bftheta^{\dagger} (\mathbf\frakq)  \bfK_\bfd\, \bfM_\bftheta^{-1}(\mathbf\frakq)\mathbf\frakp},
\end{equation}
where $\bfB_\bftheta^{\dagger}(\mathbf\frakq) = \left(\bfB_\bftheta^{\top}(\mathbf\frakq)\bfB_\bftheta(\mathbf\frakq)\right)^{-1}\bfB_\bftheta^{\top}(\mathbf\frakq)$ is the pseudo-inverse of $\bfB_\bftheta(\mathbf\frakq)$, $\bfK_\bfd = \diag(k_\bfv \bfI, k_\bfomega\bfI)$ is a damping gain with positive terms $k_\bfv$, $k_\bfomega$, and  $\bfe(\mathbf\frakq,\mathbf\frakq^*)$ is the error between $\mathbf\frakq$ and $\mathbf\frakq^*$:
\begin{equation}
\label{eq:coordinate_error}
\bfe(\mathbf\frakq,\mathbf\frakq^*) = \begin{bmatrix} \bfe_{\bfp}(\mathbf\frakq,\mathbf\frakq^*) \\ \bfe_{\bfR}(\mathbf\frakq,\mathbf\frakq^*) \end{bmatrix} = \begin{bmatrix} k_\bfp \bfR^\top(\bfp - \bfp^*) \\ \frac{1}{2}k_{\bfR}\prl{\bfR^{*\top}\bfR-\bfR^\top\bfR^{*}}^{\vee}\end{bmatrix}.
\end{equation}

\begin{lemma} \label{lem:lyap_ham}
If the input gain matrix $\bfB_\bftheta(\bf\frakq)$ of the system in \eqref{eq:ols_nln} is invertible, the control policy $\bfu = \bfpi (\bfx, \bfx^*)$ in \eqref{eq:stab_ctrl_law} always exists and asymptotically stabilizes the system to an arbitrary reference $\bfx^* = (\frakq^*, \bf0)$ with Lyapunov function given by the desired Hamiltonian $\calH_d(\bfx,\bfx^*)$ in \eqref{eq:desired_hamiltonian}.
\end{lemma}

 \begin{proof}
     See Appendix~\ref{app:proof_lyap_ham}.
 \end{proof}


Next, given the closed-loop system:
\begin{equation} \label{eq:cls_nln}
\begin{aligned}
\dot{\bfx} &= \bff_\bftheta(\bfx) + \bfG_\bftheta(\bfx) \bfpi \prl{\bfx, \bfx^*}, \quad \bfx(t_0) = \bfx_0,
\end{aligned}
\end{equation}
we derive conditions on the initial state $\bfx_0$ under which the position $\bfp$ converges to $\bfp^*$ safely, remaining in the safe set $\calF$. We first define a dynamic safety margin (DSM) $\Delta E(\bfx, \bfx^*)$ for the Hamiltonian dynamics \eqref{eq:cls_nln}:
 
\begin{equation} \label{eq:pbf}
	\Delta E(\bfx, \bfx^*) \coloneqq \bar{d}{\,}^2(\bfp^*, \calO) 
	- 2 \mathcal{H}_d \prl{\bfx, \bfx^*} / k_\bfp, 
\end{equation} 
where $\bar{d}{\,}^2(\bfp^*, \calO)$ is the truncated distance to the unsafe set $\calO$ in \eqref{eq:dyO}. Given a fixed desired point $\bfx^*$, the DSM function measures the the trade-off between safety, measured by $\bar{d}{\,}^2(\bfp^*, \calO)$, and system energy, measured by $\mathcal{H}_d \prl{\bfx, \bfx^*}$ and allows us to find a positively forward invariant set $\calS(\bfx, \bfx^*)$ such that for any $\bfx_0 \in \calS(\bfx, \bfx^*)$, the position $\bfp$ converges to $\bfp^*$ while remaining in the safe set $\calF$.

\begin{proposition} \label{prop:safe_stab}
Controller~\eqref{eq:stab_ctrl_law} renders $\calS(\bfx, \bfx^*) \coloneqq \crl{\bfx \mid \Delta E(\bfx, \bfx^*)  \geq 0}$ a positively forward invariant set for system~\eqref{eq:cls_nln}. Furthermore, if $\bfx_0 \in \calS(\bfx, \bfx^*)$, then the state $\bfp(t)$ converges to $\bfp^*$ asymptotically without collisions, i.e., $d(\bfp(t), \calO) \geq 0$ for all $t \geq t_0$.
\end{proposition}

 \begin{proof}
 See Appendix~\ref{app:proof_safe_stab}.
 \end{proof}

Using the results for a fixed $\bfx^*$ in this section, we develop a reference governor system to adaptive change $\bfx^*$ over time in Sec. \ref{sec:rg_ctrl} so that the robot can safely track a desired path.

\subsection{Reference governor design}
\label{sec:rg_ctrl}

We introduce a virtual system, called a \emph{reference governor} \citep{RG_bemporad1998reference} to adaptively track the path $\bfr$ defined in Sec.~\ref{sec:problem} and provide a time-varying reference $\bfx^*(t)$ for the actual system in \eqref{eq:cls_nln}. The motion of the governor system needs to be regulated to balance the energy of the Hamiltonian system with the distance to the unsafe set $\calO$, keeping the safety margin in \eqref{eq:pbf} positive. We define the governor as a first-order linear time-invariant system with state $\bfg(t) \in \bbR^3$ and dynamics:
\begin{equation} \label{eq:governor}
\dot{\bfg} = -k_g \prl{\bfg - \bfu_\bfg}, \quad \NEWZL{k_g > 0}.
\end{equation}
%
The governor input $\bfu_\bfg$ will be chosen to move the governor system along the reference path without violating the safety condition $\Delta E(\bfx, \bfx^*)  \geq 0$ obtained in Proposition~\ref{prop:safe_stab}. Define a \emph{local safe set} $\LS(\bfx,\bfg)$ as a region around the governor state $\bfg$ that does not violate safety:
\[
\LS(\bfx, \bfg) \coloneqq \crl{ \bfq \in \bbR^{\nOutput} \mid \norm{\bfq - \bfg}^2 \leq (1+\epsilon)^{-1} \Delta E(\bfx,\bfx^*)},
\]
where $\epsilon > 0$ is arbitrarily, ensuring that $\LS(\bfx,\bfg) \subseteq \intF$. The size of the local safe set determines how fast the governor can move along the reference path without endangering safety.

\begin{definition} \label{def:lpg}
A \emph{local projected goal} at system-governor state $(\bfx, \bfg)$ is a point $\lpg \in \LS(\bfx,\bfg)$ that is furthest along the reference path $\bfr$:
\begin{equation} \label{eq:lpg}
	\lpg = \bfr(\sigma^*),  \;\; \sigma^* = \argmax_{\sigma \in [0,1]} \crl{ \sigma \mid  \bfr(\sigma) \in \LS(\bfx,\bfg)}.
\end{equation}	
\end{definition}

Choosing the governor input as $\bfu_{\bfg} = \bar{\bfg}$ forces the governor to track the reference path adaptively, taking the safety condition $\Delta E(\bfx, \bfx^*)  \geq 0$ into account. Given the local projected goal $\bar{\bfg}(t) \in \bbR^3$ and the governor state $\bfg(t) \in \bbR^3$, we also generate a desired reference state $\bfx^*(t)$ for the system in \eqref{eq:cls_nln} by lifting $\bfg(t)$ to $\bbR^{18}$. We may choose $\bfg(t)$ as the desired position with zero desired velocity but, to provide guidance on the $SE(3)$ manifold, we need to also generate a desired orientation $\bfR^*(t)$. We construct a lifting function $\bfell: \calF \times \intF \mapsto \bbR^{18}$ to obtain $\bfx^* = \bfell(\bfg,\bar{\bfg})$ as:
\begin{equation} \label{eq:lift_funcs}
\bfell(\bfg, \lpg) = \begin{bmatrix}
\bfp^{*\top} & \bfr_1^{*\top} & \bfr_2^{*\top} &\bfr_3^{*\top}, \bf0^\top, \bf0^\top
\end{bmatrix}^\top,
\end{equation}
where $\bfp^* = \bfg$ and $\bfr_1^*$, $\bfr_2^*$, $\bfr_3^*$ are the rows of the matrix:
\begin{equation} \label{eq:rotRg}
	\bfR^*(\bfg, \lpg) = 
	\begin{cases}
	\bfI  & \text{if } \bfe_3 \times \bfc_1 = 0\\
	\brl{\bfc_1\;\; \bfc_2 \;\; \bfc_3} & \text{otherwise},
	\end{cases}
\end{equation}
with $\bfe_3 = \brl{0, 0, 1}^\top$, $\bfc_1 = (\lpg - \bfg) / \norm{\lpg - \bfg}$, $\bfc_2 = (\bfe_3 \times \bfc_1) / \norm{\bfe_3 \times \bfc_1}$ and $\bfc_3 = (\bfc_1 \times \bfc_2) / \norm{\bfc_1 \times \bfc_2}$. If $\bfg = \lpg$, the most recent backup of $\bfR^*$ or $\bfI$ may be used. 

Our safe tracking control design is visualized in Fig.~\ref{fig:structure}. It consists of two parts: 1) a first-order reference governor system with state $\bfg$ adaptively following the local projected goal $\lpg$ along the path $\bfr$ and 2) a closed-loop Hamiltonian system tracking the reference signal $\bfx^* = \bfell(\bfg, \lpg)$. Our main result is summarized in the following theorem.

\begin{figure}[t]
\centering
\begin{subfigure}
    \centering
    \includegraphics[height=37mm]{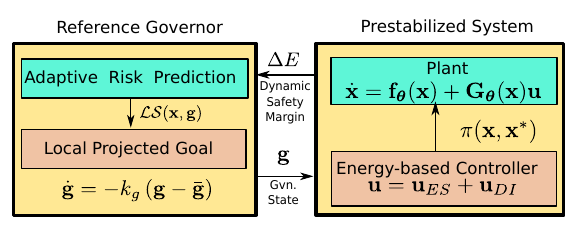}
\end{subfigure}%
\hfill%
\begin{subfigure}
    \centering
    \includegraphics[height=35mm]{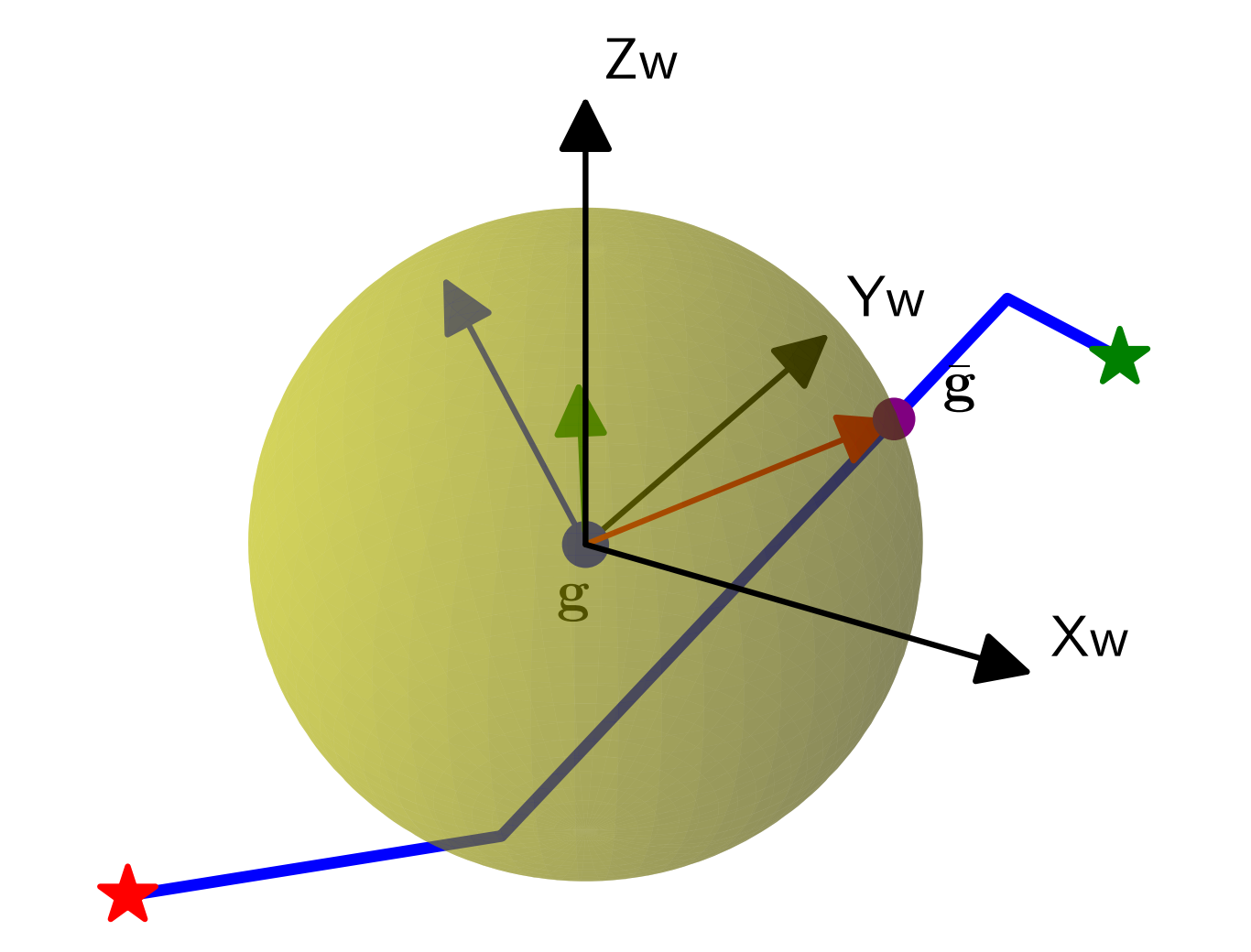}
\end{subfigure}%
\caption{(a) Structure of the reference-governor tracking controller: a reference governor with state $\bfg$ adaptively tracks a point $\bar{\bfg}$ along the desired path $\bfr$ and generates a time-varying equilibrium $\bfx^* = \bfell(\bfg, \lpg)$ for the closed-loop Hamiltonian system; (b) A local projected goal $\lpg$ (purple dot) is generated as the furthest intersection between the local safe set $\LS(\bfx, \bfg)$ (yellow sphere) and the path $\bfr$ (blue curve). Given $\bfg$ and $\lpg$, a desired equilibrium $\bfx^* = \bfell(\bfg, \lpg)$ is generated for the Hamiltonian system with orientation indicated by the red, green, and blue arrows, respectively.}
\label{fig:structure}
\end{figure}

\begin{theorem}\label{thm:main_result}
Given a reference path $\bfr$, consider the closed-loop Hamiltonian system in \eqref{eq:cls_nln} and the closed-loop reference governor system, $\dot{\bfg} = -k_g \prl{\bfg - \bar{\bfg}}$, with local projected goal $\bar{\bfg}$ provided in Definition~\ref{def:lpg}. Suppose that the initial state $(\bfx_0, \bfg_0)$ satisfies:
	\begin{equation} \label{eq:thm_ic}
	\Delta E \prl{\bfx_0, \bfell(\bfg_0, \lpg_0)} > 0,\quad  \bfg_0 = \bfr(0) = \bfp(t_0) \in \intF,
	\end{equation}
where $\Delta E(\bfx, \bfx^*)$ is a dynamic safety margin defined in \eqref{eq:pbf}. Then, the joint system state $(\bfx, \bfg)$ converges to $(\bfell(\bfr(1), \bfr(1)), \bfr(1))$ without violating the constraints, i.e., $\bfp(t) \in \calF$, $\forall t \geq t_0$.
\end{theorem}

\begin{proof}
See Appendix~\ref{app:proof_safe_output_tracking}.
\end{proof}



\section{Evaluation}

\begin{figure}[t] 
	\centering
	\includegraphics[width=0.49\linewidth]{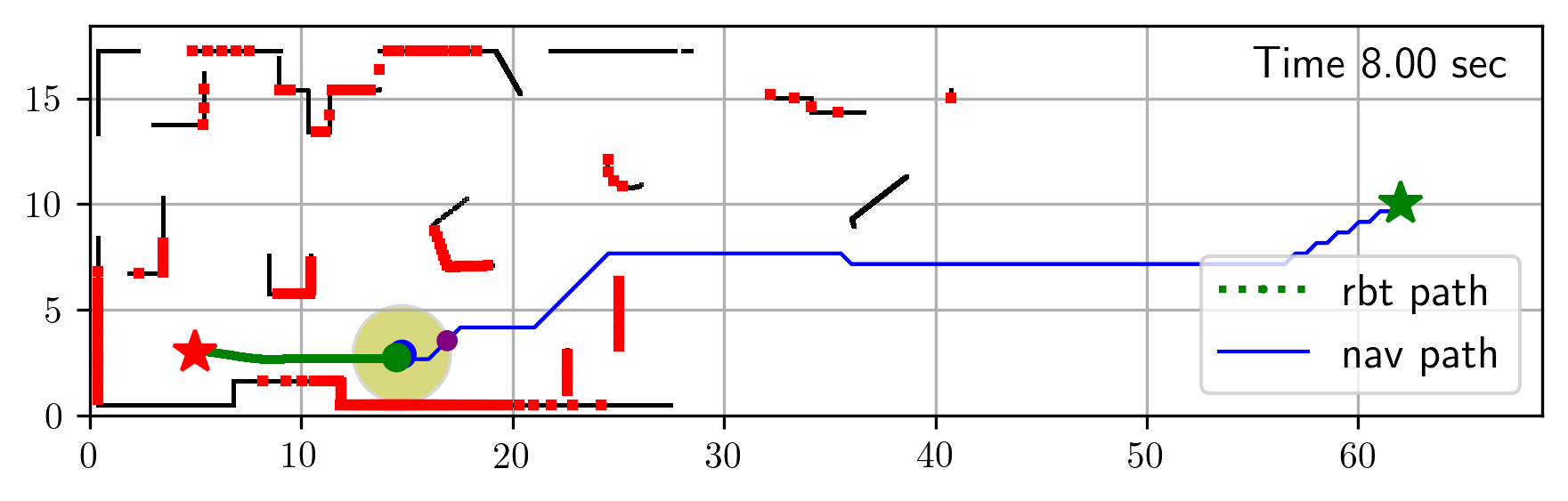}%
	\hfill%
	\includegraphics[width=0.49\linewidth]{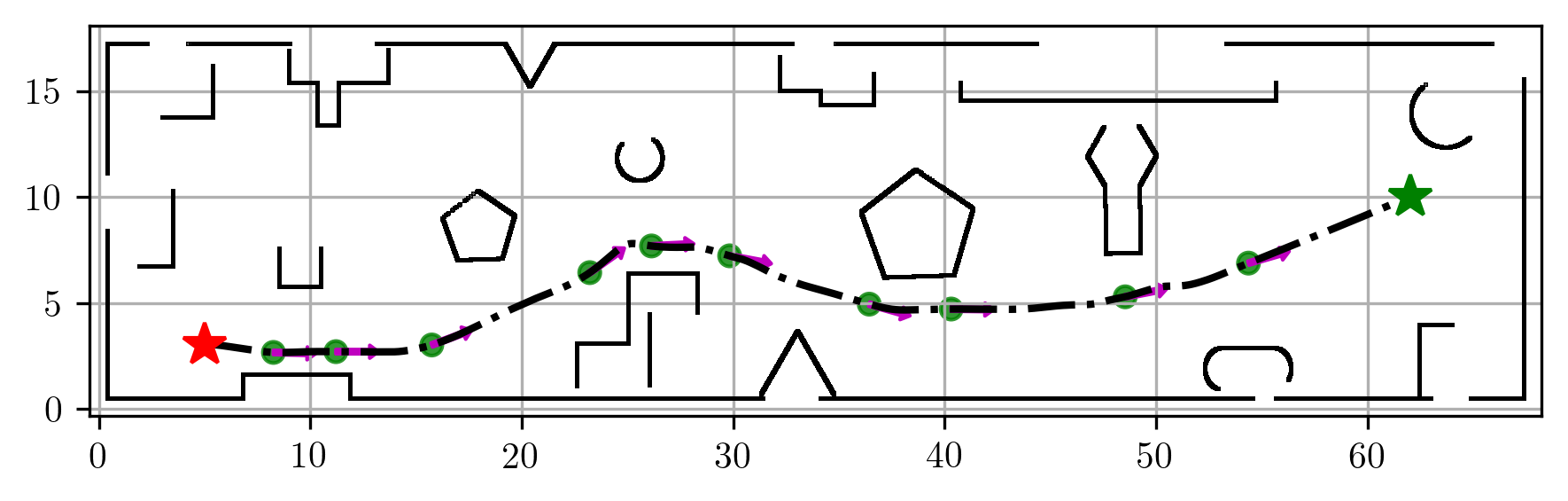}%
	\\
	\includegraphics[width=0.49\linewidth]{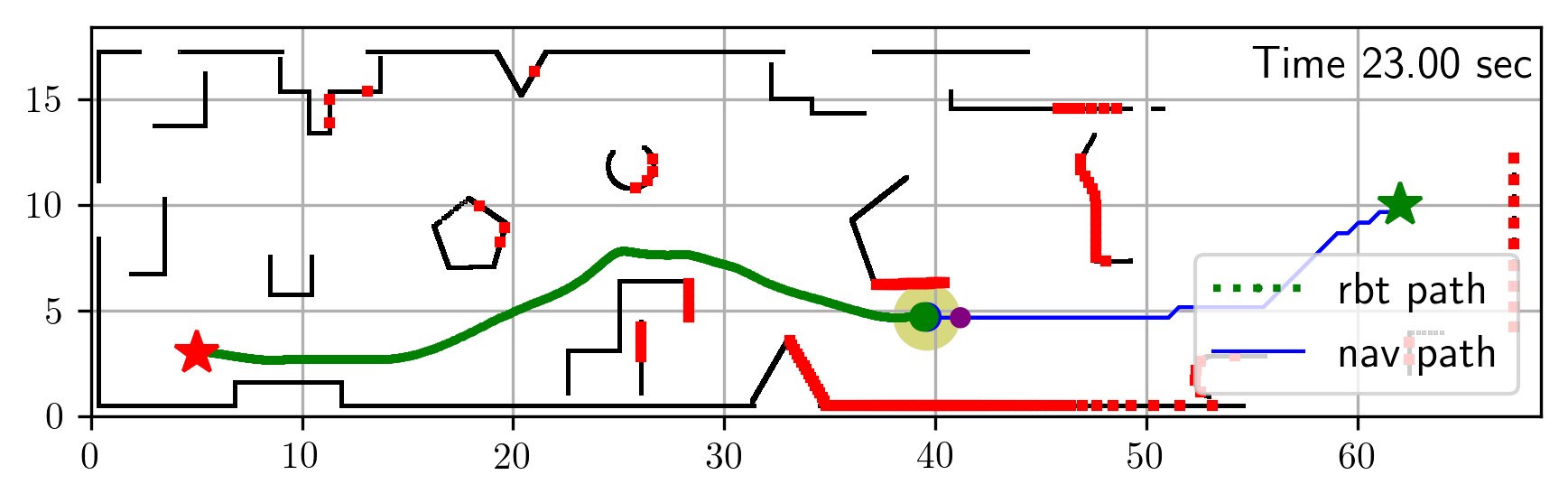}%
	\hfill%
	\includegraphics[width=0.49\linewidth]{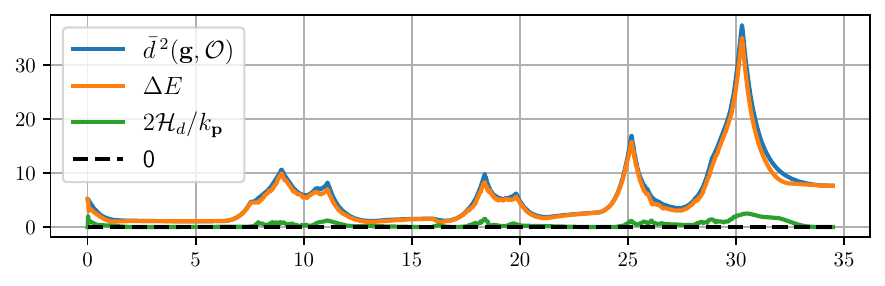}%
	\\
	\includegraphics[width=0.49\linewidth]{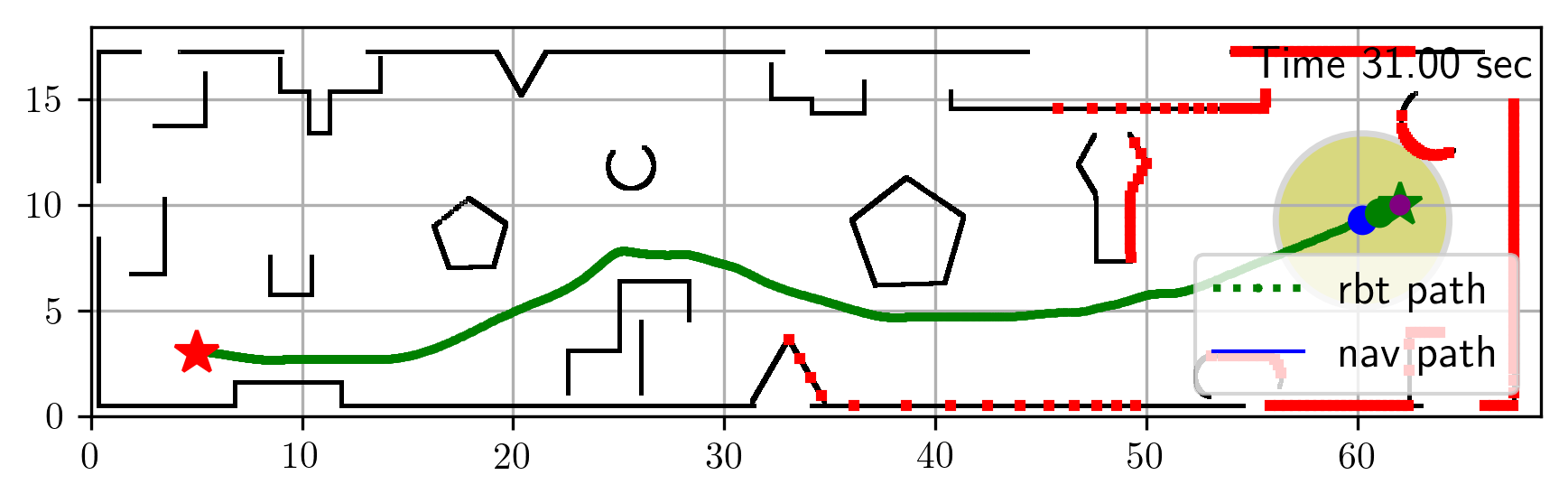}%
	\hfill%
	\includegraphics[width=0.49\linewidth]{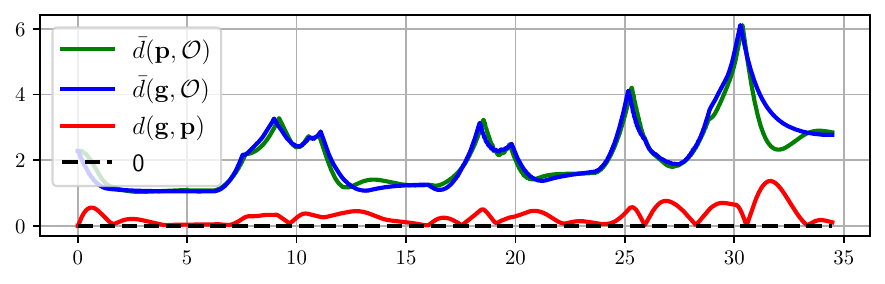}%
	\caption{Safe tracking control for a hexarotor aerial robot in an unknown cluttered environment. The hexarotor (green dot) navigates from a start location (red star) to a goal location (green star) while avoiding obstacles. The obstacles are sensed by a simulated LiDAR sensor (red points). On the left are snapshots of the environment showing the robot (green dot), lidar scans (red dots), and the unsafe set $\calO$ (black surfaces) at different times. The reference path (blue curve) is re-planned online from the governor position (blue dot) to the goal (green star) using an $A^*$ algorithm. The local projected goal $\lpg$ (purple dot) is computed based on the obstacle distance (gray ball) and the local safe set (yellow ball).
	On the upper right, the robot's heading (purple arrow) are plotted along the path. The middle right figure shows the dynamic safety margin $\Delta E$ in \eqref{eq:pbf} and the scaled desired Hamiltonian $2 \calH_d(\bfx, \bfx^*)/k_\bfp$ while the lower right figure plots the distance to the obstacles $d(\bfp(t), \calO)$ (green curve), showing that the safety constraints are never violated.}
	\label{fig:hexarotor_sim_mesh}
\end{figure}

\begin{figure}[th]
\centering
\begin{subfigure}
        \centering
        \includegraphics[width=0.45\linewidth,trim=10mm 10mm 10mm 10mm, clip]{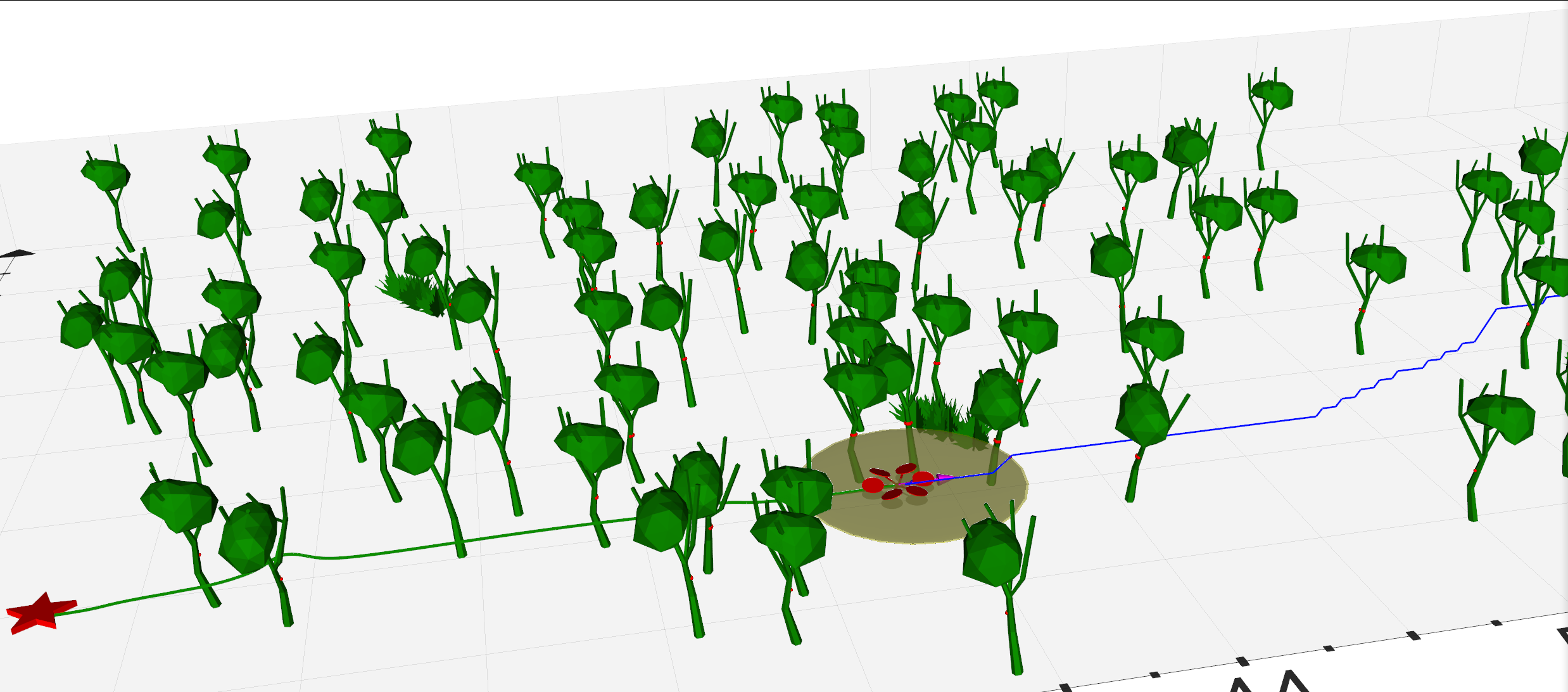}
\end{subfigure}%
\hfill
\begin{subfigure}
        \centering
        \includegraphics[width=0.45\linewidth,trim=10mm 10mm 10mm 10mm, clip]{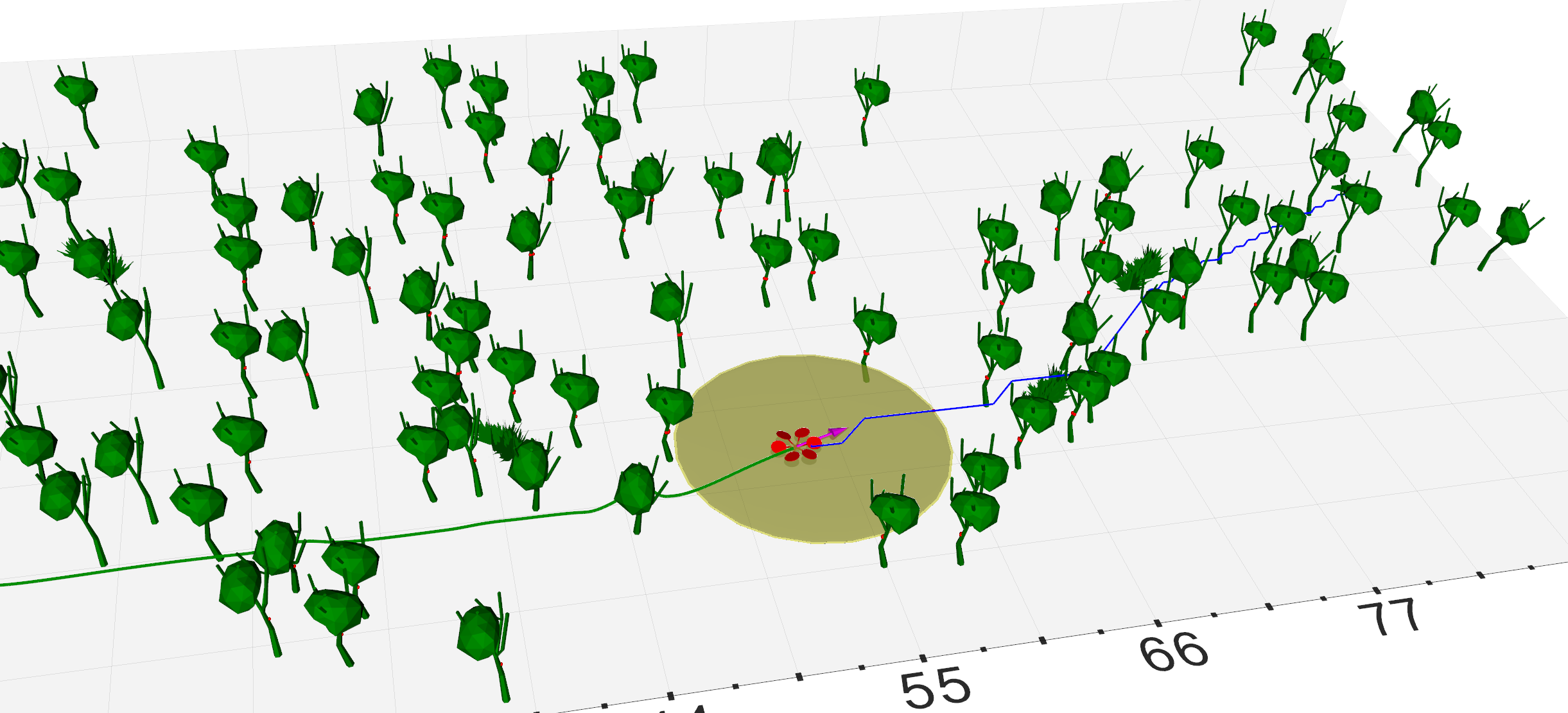}
\end{subfigure}%
\\
\begin{subfigure}
        \centering
        \includegraphics[height=20mm]{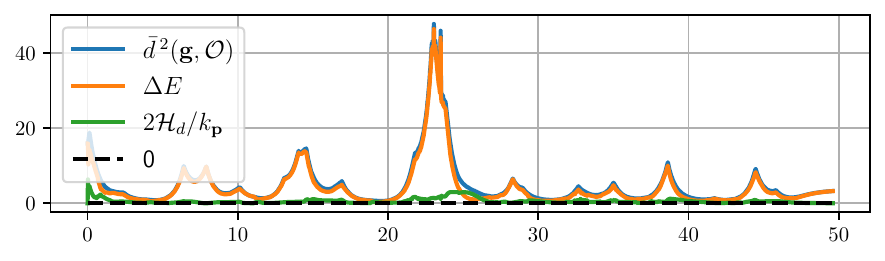}
\end{subfigure}%
\hfill
\begin{subfigure}
        \centering
        \includegraphics[height=20mm]{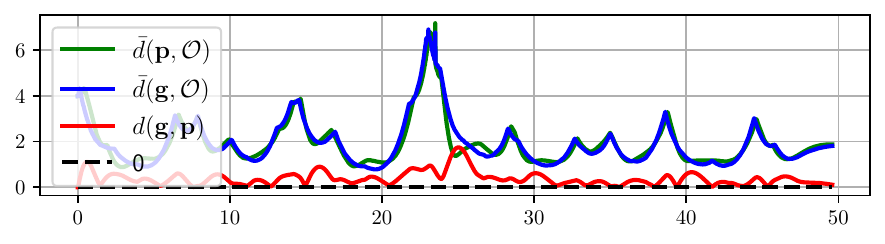}
\end{subfigure}%
\caption{Safe navigation of a hexarotor robot in a forest is demonstrated in (a) and (b). The bottom plots (c) and (d) show the dynamic safety margin $\Delta E$, the scaled desired Hamiltonian $2 \calH_d(\bfx, \bfx^*)/k_\bfp$, and the distance to the obstacles $d(\bfp(t), \calO)$, indicating that the safety constraints are never violated.}
\label{fig:sim_forest}
\end{figure}

This section evaluates our safe tracking controller on a simulated hexarotor UAV using the learned Hamiltonian dynamics in Sec. \ref{subsec:ham_dyn_learning_evaluation}. 
The task is to navigate from a start  \NEWZL{position} to a goal in an environment without colliding with the obstacles. \NEWZL{To guarantee stability and safety, all control gains must be positive.} The following control gains were used with the regulation controller in Sec. \ref{sec:ham_control}: $k_p = 0.25$, $k_\bfR = 125\bfJ$, $k_v = 0.125$, $k_\bfomega = 10\bfJ$ in \eqref{eq:desired_hamiltonian} and $k_\bfg = 1.0$ with the governor control in \eqref{eq:governor}. 
We test the controller in a challenging unknown 3D environment with complex obstacles $\calO$. A simulated LiDAR scanner provides point cloud measurements $\calP(t) := \crl{\bfy_i(t)}$ of the surface of the unsafe set $\calO$, depending on the system pose, with a maximum sensing range of $\beta = 30$. The distance from the governor $\bfg(t)$ to the unsafe set $\calO$ is approximated via $\bar{d}(\bfg(t); \calO) \approx \min_{\bfy \in \calP(t)} \norm{\bfg(t) - \bfy}$. The point clouds $\calP(t)$ are used to construct an occupancy grid map online and a reference path $\bfr$ is replanned periodically using the A* algorithm to ensure that $\bfr(\sigma) \in \intF$. 
\NEWZL{
In this paper, we assume the learned system dynamics are accurate and simulated a noise-free environment. We leave model uncertainty, measurement noise, and external disturbances for future work.
}

Fig.~\ref{fig:hexarotor_sim_mesh} and Fig.~\ref{fig:sim_forest} show the behavior of the closed-loop hexarotor system in \NEWZL{two different} environments. The reference governor follows the projected goal $\bar{\bfg}$ and generates a time-varying equilibrium $\bfx^* = \bfell(\bfg,\bar{\bfg})$ for the hexarotor. The dynamic safety margin $\Delta E(\bfx,\bfx^*)$ fluctuates during this process \NEWZL{but it never becomes negative as seen in the figures}. The augmented system $(\bfx, \bfg)$ is controlled adaptively, slowing down when the dynamic safety margin decreases (e.g., when the robot is close to an obstacle and has large total energy $\calH_d$) and speeding up otherwise (e.g., when the robot is far away from the obstacles or has small total energy $\calH_d$). The simulations show that our control policy successfully drives the system from the start to the end of the reference path while avoiding sensed obstacle online, i.e., $d(\bfp, \calO)$ remains positive throughout the motion.

\section{Conclusion}
\NEWZL{
This paper developed a tracking controller for fully-actuated Hamiltonian systems which enables safe autonomous navigation in unknown environments. 
}
Given only a training set of system state-control trajectories, our approach estimates the system dynamics \NEWZL{accurately using a neural ODE network}
and synthesizes a controller that avoids obstacles based on run-time distance measurements. 
Our method was demonstrated on a simulated hexarotor aerial robot navigating in complex 3D environments. Future work will focus on capturing model uncertainty and external disturbances in the design and deploying it on a hardware platform.

\acks{We gratefully acknowledge support from NSF RI IIS-2007141 and NSF CCF-2112665 (TILOS).}

\appendix
\section{\,}
\label{app:proof_lyap_ham}
\begin{proof}
Let $\bfp_e = \bfp - \bfp^*$ be the position error and 
$\bfR_e \coloneqq \bfR^{*\top} \bfR = 
\begin{bmatrix}
	\bfr_{e1} & \bfr_{e2} & \bfr_{e3}
\end{bmatrix}^\top$ 
be the rotation error, then the error state can be expressed as $\bfx_e \coloneqq (\mathbf\frakq_e, \mathbf\frakp_e)$, where
\begin{equation*}
    \mathbf\frakq_e = \begin{bmatrix} (\bfp - \bfp^*)^\top\quad 	\bfr_{e1}^\top \quad \bfr_{e2}^\top \quad \bfr_{e3}^\top \end{bmatrix}^\top, \quad \mathbf\frakp_e = \mathbf\frakp.
\end{equation*}
Since the system is fully-actuated, \NEWZL{i.e.,} the input matrix $\bfB_\bftheta(\bf\frakq)$ is invertible. The controller in \eqref{eq:stab_ctrl_law} exists and the resulting closed-loop error dynamics becomes \citep{duong21hamiltonian}:
\begin{equation}
\label{eq:desired_port_Hal_dyn_tracking}
\begin{bmatrix}
\dot{\mathbf\frakq}_e \\
\dot{\mathbf\frakp}_e \\
\end{bmatrix}
= \begin{bmatrix}
\bf0 & \bfJ \\
-\bfJ^\top & -\bfK_d\end{bmatrix}
\NEWZL{
\begin{bmatrix}
\nabla_{\frakq_e} \calH_d \\
\nabla_{\frakp_e} \calH_d 
\end{bmatrix}
},
\qquad \bfJ =\begin{bmatrix}
	\bfR^\top\!\!\!\! & \bf0 & \bf0 & \bf0 \\
	\bf0 & \hat{\bfr}_{e1}^\top & \hat{\bfr}_{e2}^\top & \hat{\bfr}_{e3}^\top
	\end{bmatrix}^\top.
\end{equation}
Using group property, rotation error matrix $\bfR_e = \bfR^{*\top}\bfR \in SO(3)$, therefore, $\bfR_e$ is a orthogonal matrix. All columns of $\bfR_e$ are orthonormal and all elements in $\bfR_e$ are less than $1$, hence, $\tr\prl{\bfI - \bfR^{*\top} \bfR} \geq 0$.
Since $\bfM$ is a positive definite matrix, it is easy to see that $\calH_d$ is positive definite, and $0$ minimum value is achieved only at $\bfx^*_e = (\frakq_e, \bf0)$ with $\frakq_e = [\bf0^\top, \bfe_1^\top, \bfe_2^\top, \bfe_3^\top]^\top$.
The time derivative can be computed as:
%
\NEWZL{
\begin{equation}
\begin{aligned}
\dot{\calH}_d(\bfx, \bfx^*) &= \nabla_{\frakq_e} \calH_d^\top\dot{\mathbf\frakq}_e + \nabla_{\frakp_e} \calH_d ^\top \dot{\mathbf\frakp}_e\\
              &= \nabla_{\frakq_e} \calH_d^\top \bfJ \nabla_{\frakp_e} \calH_d - \nabla_{\frakp_e} \calH_d^\top \bfJ^\top  \nabla_{\frakq_e} \calH_d - \nabla_{\frakp_e} \calH_d^\top \bfK_d \nabla_{\frakp_e} \calH_d\\
              & = -\frakp_e^\top \bfM^{-1}(\frakq)\bfK_d \bfM^{-1}(\frakq) \frakp_e.
\end{aligned}
\end{equation}
}
Hence, $\dot{\calH}_d(\bfx, \bfx^*) \leq 0$ for all $\bfx_e$. It is not hard to show that the only point can stay within set 
$\crl{\dot{\calH}_d = 0}$ is at origin. By the LaSalle's invariance principle \citep{khalil2002nonlinear}, the system  \eqref{eq:desired_port_Hal_dyn_tracking} asymptotically stabilizes to desired equilibrium $\bfx^*_e$, i.e. $\bfx^* = (\frakq^*, \bf0)$.
\end{proof}
%
%
\section{\,}
\label{app:proof_safe_stab}
\begin{proof}
From Lemma~\ref{lem:lyap_ham}, we know that for any constant $\bfx^* = (\frakq^*, \bf0)$, $\calH_d(\bfx, \bfx^*)$ is a Lyapunov function to certificate stability of $\bfx^*$. Hence, 
\begin{equation*}
    \calS(\bfx, \bfx^*) = \crl{\bfx \mid \Delta E(\bfx, \bfx^*) \geq 0} = \crl{\bfx \mid H_d(\bfx, \bfx^*) \leq \frac{2 \bar{d}{\,}^2(\bfp^*, \calO)}{k_\bfp}}
\end{equation*}
is forward invariant and control signal $\bfpi \prl{\bfx, \bfx^*}$ can steer the state $\bfp$ of the system towards $\bfp^*$.  It remains to show that during the convergence, constraints over $\bfp(t)$ is never violated. In the proof of Lemma~\ref{lem:lyap_ham}, we have shown that the second term of desired Hamiltonian~\eqref{eq:desired_hamiltonian}, $\frac{1}{2} k_{\bfR}\tr(\bfI - \bfR^{*\top}\bfR)$ is non-negative for all $\bfR^* \in SO(3)$, therefore, 
$\frac{2}{k_\bfp} \mathcal{H}_d \prl{\bfx, \bfx^*} \geq (\bfp - \bfp^*)^\top (\bfp - \bfp^*).$
From \eqref{eq:dyO} and above inequality, we know that $\bfp(t) \in \calF$ for all $t \geq t_0$, since
\begin{equation*}
d^2(\bfp^*, \calO) \geq \bar{d}\,^2(\bfp^*, \calO) \geq \frac{2}{k_\bfp} \mathcal{H}_d \prl{\bfx_0, \bfx^*} \geq \frac{2}{k_\bfp} \mathcal{H}_d \prl{\bfx(t), \bfx^*} \geq d^2 \prl{\bfp^*, \bfp(t)}.
\end{equation*}
\end{proof}

\section{\,}
\label{app:proof_safe_output_tracking}
\begin{proof}
    For conciseness, $\Delta E(t) = \Delta E \prl{\bfx(t), \bfell(\bfg(t), \lpg(t))}$. 
	Initially, $\bfg_0 = \bfp_0 = \bfr(0) \in \LS(\bfx_0,\bfg_0)$ and $\Delta E(t_0) > 0$, the local projected goal $\lpg$ and associated rotation matrix $\bfR^*(\bfg, \lpg)$ is well defined. As $\lpg$ moves along the reference path $\bfr$, i.e., path parameter $\sigma$ in \eqref{eq:lpg} increases. While $\bfg$ chasing $\lpg$ by \eqref{eq:governor}, system state $\bfx$ tracks $\bfx^* = \bfell(\bfg, \lpg)$ using controller $\bfpi(\bfx, \bfx^*)$ in \eqref{eq:stab_ctrl_law}.
	During this process, the safety margin $\Delta E(t)$ is fluctuating which regulates behavior of $\bfg$ through $\lpg$. Since the system dynamics are continuous, $\Delta E(t)$ cannot become negative without crossing $0$ from above at some time $T_0$. As $\Delta E(t) \downarrow 0$, the local safe zone will shrink to a point, i.e., $\LS(\bfx, \bfg) \downarrow \crl{\bfg}$ . This immediately stops the movement of governor because $\lpg = \bfg(T_0)$ and $\dot{\bfg}(T_0) = 0$. 
	From Proposition~\ref{prop:safe_stab}, we know that $\bfx(t)$ will stay within $\calS(\bfx, \bfx^*(T_0))$ for $t \geq T_0$ and position constraints will not be violated. As $\bfx \rightarrow \bfx^*(T_0)$ because $\dot{H}_d(\bfx, \bfx^*(T_0)) < 0$ when $\bfg$ is static and $\bfx \neq \bfx^*(T_0)$, there exists $h > 0$ such that $\Delta E(T_0 + h)$ becomes strictly positive. Hence, the governor is able to  move again towards new $\lpg$ getting further along the path as discussed previously. This process continues until the augmented system stabilized at $\prl{\bfell \prl{\bfr(1), \bfr(1)}, \bfr(1)}$ where $\lpg$ stops changing.
\end{proof}

\bibliography{bib/thai_zhl}
\end{document}